\documentclass{article}
\PassOptionsToPackage{numbers, compress}{natbib}
     \usepackage[preprint]{neurips_2020}

     \usepackage{neurips_2020}

\usepackage[utf8]{inputenc} % allow utf-8 input
\usepackage[T1]{fontenc}    % use 8-bit T1 fonts
\usepackage{hyperref}       % hyperlinks
\usepackage{url}            % simple URL typesetting
\usepackage{booktabs}       % professional-quality tables
\usepackage{amsfonts}       % blackboard math symbols
\usepackage{nicefrac}       % compact symbols for 1/2, etc.
\usepackage{microtype}      % microtypography
\usepackage{amsmath}
\usepackage{subfig}
\usepackage{graphicx}
\usepackage{amsthm}
\usepackage[utf8]{inputenc}
\usepackage{textcomp}

\newtheorem{theorem}{Theorem}
\newtheorem{remark}{Remark}
\newtheorem{example}{Example}

\usepackage{natbib}
\usepackage{amssymb}
\usepackage{ulem}
\usepackage[ruled,vlined]{algorithm2e}

\usepackage{xcolor}
\newcommand{\fl}[1]{\textcolor{blue}{\bf\small [#1 --FL]}}
\newcommand{\todo}[1]{\textcolor{green}{\bf\small [#1 ]}}

\title{The Penalty Imposed by Ablated Data Augmentation}

\author{%
  Frederick Liu \\
  Google\\
  \texttt{frederickliu@google.com} \\
  \And
  Amir Najmi \\
  Google\\
  \texttt{amir@google.com} \\
  \And
  Mukund Sundararajan \\
  Google\\
  \texttt{mukunds@google.com}
}

\begin{document}

\maketitle

\begin{abstract}
  There is a set of data augmentation techniques that ablate parts of the input at random. These include input dropout~\cite{srivastava2014dropout}, cutout~\cite{devries2017improved}, and random erasing~\cite{Zhong2017RandomED}. We term these techniques \emph{ablated data augmentation}. Though these techniques seems similar in spirit and have shown success in improving model performance in a variety of domains, we do not yet have a mathematical understanding of the differences between these techniques like we do for other regularization techniques like L1 or L2.   
  First, we study a formal model of mean ablated data augmentation and inverted dropout for linear regression. We prove that ablated data augmentation is equivalent to optimizing the ordinary least squares objective along with a penalty that we call the Contribution Covariance Penalty and inverted dropout, a more common implementation than dropout in popular frameworks~\cite{tensorflow, pytorch}, is equivalent to optimizing the ordinary least squares objective along with Modified L2. 
  
  For deep networks, we demonstrate an empirical version of the result if we replace contributions with attributions~\cite{STY17} and coefficients with average gradients, i.e., the Contribution Covariance Penalty and Modified L2 Penalty drop with the increase of the corresponding ablated data augmentation across a variety of networks.  
\end{abstract}

\section{Motivation and Related Work}
\label{related work}

Data augmentation techniques increase the training data available to the model without collecting new labelled data. The idea is to take a labelled data point, perform a  modification that does not change the meaning of the data point and add it to the training data.

Data augmentation is popular for vision tasks. One approach to data augmentation is to apply standard image processing transformations like shifting, scaling, stretching, rotation, blurring, color and intensity modification, etc.~\cite{bengio2011deep, krizhevsky2012imagenet, lecun1998gradient, wu2015deep}. There have also been proposals to learn the transformation that is applied to the data ~\cite{lemley2017smart, cubuk2019autoaugment}. Let us call these \emph{transformation based data augmentation} techniques.

A second class data augmentation techniques work by ablating parts of the input. We term these techniques \emph{ablated data augmentation} techniques. Some data augmentation 
work by ablating parts of the input at random. These include inverted input dropout~\cite{srivastava2014dropout} that randomly sets part of the input to zero and scales input at test time, and cutout~\cite{devries2017improved}, and random erasing~\cite{Zhong2017RandomED} that erase parts of an image, and switchout~\cite{wang2018switchout} that randomly replaces words in both
the source sentence and the target sentence
with other random words from their corresponding vocabularies for neural machine translation. 

The transformation based techniques have an obvious justification: We expect the model to be \emph{invariant} to the transformations, and therefore we have expanded our dataset without collecting new labels. On the other hand, the ablated data augmentation techniques will (randomly) destroy signal. It is less clear that they can be justified by invariance. 

Informally, ablated data augmentation methods are considered a regularization techniques, i.e., it is seen as a method to reduce model complexity (see, for instance~\cite{devries2017improved}). \emph{However, do specific ablated data augmentation optimize a specific form of penalty? And if so, what kind of penalty? How does the penalty vary across different types of ablated data augmentation techniques?}  

We do know such penalties for other popular regularization techniques: Ridge Regression optimizes squared error with an L2 penalty on the coefficients, whereas Lasso regression optimizes squared error L1 penalty. That said, we would expect the story to be messier for ablated data augmentation. Unlike these techniques, which are defined by their optimization objectives, the regularizing action is a \emph{consequence} of a data manipulation procedure.
However, there is precedent for such a results:~\cite{Bishop95} shows that training a network with noise added to the input is equivalent to using a Tikhonov regularizer (i.e., Ridge Regression).  

Section 9.1 of the dropout paper~\cite{srivastava2014dropout} shows that dropout applied to linear regression is equivalent to a form of Ridge Regression. They argue that the weights of features with high variance are squeezed by increasing the dropout rate. However, this discussion does not take the effect of the variation in feature means and feature covariances. Also, it does not analyze inverted dropout, which is more commonly used and is the standard implementation in popular frameworks~\cite{tf_dropout, pytorch_dropout}.

%\todo{Amir: Is Diag $X*X^T$ the same as $v_j + \mu_j^2$? Yes it is} 

\subsection{Our Results}

\begin{itemize}
    \item We prove (see Theorem~\ref{thm:main1}) that mean ablated data augmentation is equivalent to optimizing the ordinary least squares objective along with a penalty that we call the \emph{Contribution Covariance Penalty}. For a linear model, a feature's contribution is the product of its coefficient and its value. Thus, a feature pair is penalized if their contributions have a negative covariance. 
    
    \item In contrast, we prove (see Theorem~\ref{thm:main2}) that inverted input dropout, is equivalent to optimizing the ordinary least squares objective along with a penalty that we call the \emph{Modified L2 penalty}, which is the standard L2 penalty scaled by the the sum of the square of the feature means and the feature variance. It is perhaps surprising that a slight in the data ablation protocol results in a very different type of regularization. 
    
    \item For neural networks, we demonstrate an empirical analogs of Theorems~\ref{thm:main1} and~\ref{thm:main2}. To do this, we replace the notion of contribution for linear models with the notion of attribution(\cite{STY17,sampledshapley,Lundberg2017AUA}) for neural networks. Just as the prediction of a linear model can be decomposed into the contributions of its features, the prediction of a neural network can be decomposed into attributions of its features (see Section~\ref{sec:IG} in appendix for details). Thereafter, we replace contribution by attribution in the definitions of the Contribution Covariance Penalty and coefficients with average gradients of the Modified L2 Penalty. We show that Modified L2 Penalty drops with increased inverted input dropout data augmentation and the Contribution Covariance Penalty drops with increased mean ablated data augmentation. We notice this trend across a variety of tasks, and for a range of network depths (see Sections~\ref{sec:learnings}). We also show that the reverse is not true.
\end{itemize}

\section{Effect of Ablated Data Augmentation on Linear Models}

In this section we identity the mathematical form of the penalty imposed by two types of ablated data augmentation for linear models. 

\subsection{Ablated Data Augmentation}

To perform data augmentation, we generate a synthetic dataset by randomly selecting (with replacement) the $i\mbox{-th}$ observation $(x_i, y_i)$ from the original data $d$ and adding to the synthetic dataset a modified version of this observation. The modification is to \textbf{ablate} each feature independently with probability $\lambda$, i.e., we replace the feature value with the dataset mean for the feature. The response $y_i$ is not ablated. We call this \textbf{mean ablation data augmentation}. We also study \textbf{inverted input dropout}, which differs from mean ablation data augmentation in two ways. First, it replaces the feature value with zero and not the mean. Second, it scales unablated features by the reciprocal of the ablation probability. In either method, this process of resampling probabilistically ablated observations is repeated $N$ times, and we denote the resultant dataset as $D^\lambda_N$. We denote the least squares solution for this data by $\hat{\beta}(D^\lambda_N)$. 

\subsection{Contribution Covariance Penalty}
\label{sec:penalty}

We will show (Theorem~\ref{thm:main1}) that the least squares model for the data augmented with mean ablation is equivalent to the optimal linear model (on the original data) with a certain regularization penalty, which we call the \emph{Contribution Covariance penalty}. The \emph{contribution} of a feature $j$ for a specific input $i$ is the product of its coefficient times its value for the input, $\beta_j x_{ij}$; recall that in a linear model the prediction is the sum of the contributions (plus the bias term $\beta_0$). Formally, the penalty is: 
\begin{align}
\label{eq:penalty}
n \sum_{\substack{j \neq k \\ j,k\in J}}
- \mathrm{cov}(\beta_j x_j, \beta_k x_k)
\end{align}
Here, the covariances are across the $n$ training data points. Note the minus sign, indicating that penalty increases if contributions are anticorrelated. 

As is usual with regularization, our minimization objective will be the sum of the least squares objective and $\lambda$ times the penalty. Though $\lambda$ can take on values in the range $[0,\infty)$, we will be concerned with values of $\lambda$ in the range $[0,1]$, because this the range of $\lambda$ necessary to mimic data augmentation. 

We denote the model that minimizes this objective by $\hat{\beta}_{\mathrm{ccp}(\lambda)}(d)$. As we show in the proof of Theorem~\ref{thm:main1}, this minimum exists and is unique. 

\subsection{Modified L2 Penalty}
\label{sec:l2penalty}

We will show (Theorem~\ref{thm:main2}) that the least squares model for the data augmented with inverted input dropout is equivalent to the optimal linear model (on the original data) with a certain regularization penalty, which we call the \emph{Modified L2 Penalty}. Let $v_j$ and $\mu_j$ be the variance and mean of feature $j$ respectively. Then, the penalty is given by: 

\begin{align}
\label{eq:L2penalty}
\sum_{j \in J} (v_j + \mu_j^2) \beta_j^2
\end{align}

We denote the model that minimizes this objective by $\hat{\beta}_{\mathrm{ml2p}(\lambda)}(d)$. If the data is standardized, i.e., if all the features have mean of zero and a variance of one, the Modified L2 Penalty coincides with the standard L2 Penalty.

\subsection{Formal Results}

\begin{theorem}[Mean Ablated Data Augmentation as Regularization]
\label{thm:main1}
The Ordinary Least Squares (OLS) model with mean ablated data augmentation $\hat{\beta}(D^\lambda_N)$ converges to the least squares model with the Contribution Covariance penalty trained on the original data set $d$, i.e., $\hat{\beta}_{\mathrm{ccp}(\lambda)}(d)$, as the augmented data set size $N \rightarrow \infty$.
\end{theorem}

\begin{proof}
We show that the closed-form solutions for the two models are identical.

First, we study the model with mean ablated data augmentation. 

Let $\tilde{X}$ denote the features which include mean ablations and $\tilde{y}$ the associated response values for the augmented dataset. Thus $D_N^\lambda = (\tilde{X}, \tilde{y})$. As explained in section \ref{sec:OLS-mean-subtraction} in appendix, the OLS regression formula can be recast in terms of features and response with their means subtracted. Let underbar represent the subtraction of column-wise means, i.e., $\underline{z} = z - \bar{z}$. Then the OLS solution on $D_N^\lambda$ can be expressed in terms of 
$\underline{\tilde{X}}^T \underline{\tilde{X}}$ and 
$\underline{\tilde{X}}^T \underline{\tilde{y}}$. As shown in section \ref{eq:ablation-asymptotics} of the appendix:

\begin{align*}
\lim_{N\to \infty} \frac{1}{N} \underline{\tilde{X}}^T \underline{\tilde{y}}
    &\overset{a.s.}{\to} \frac{1}{n} (1-\lambda) \underline{X}^T \underline{y} \\
\lim_{N\to \infty} \frac{1}{N} \underline{\tilde{X}}^T \underline{\tilde{X}} 
    &\overset{a.s.}{\to} \frac{1}{n} (1-\lambda)^2 \underline{X}^T \underline{X} + \lambda(1-\lambda) V
\end{align*}
where $V$ is the $k\times k$ diagonal matrix of feature variances, $n$ is the number of data points in the original dataset.

Applying OLS regression to $D^\lambda_N$:
\begin{align*}
\hat{\beta}(D^\lambda_N) 
&= \Big(\underline{\tilde{X}}^T \underline{\tilde{X}}\Big)^{-1} \underline{\tilde{X}}^T \underline{\tilde{y}}\\
&\overset{a.s.}{\to} \Big(\frac{1}{n} (1-\lambda)^2 \underline{X}^T \underline{X} + \lambda(1-\lambda) V\Big)^{-1}
\Big(\frac{1}{n} (1-\lambda) \underline{X}^T \underline{y}\Big)
\\
&=\Big((1-\lambda) \underline{X}^T \underline{X} + n\lambda V\Big)^{-1}\underline{X}^T \underline{y}
\end{align*}

Next, we study regularization with the Contribution Covariance penalty. We can rewrite the penalty in matrix form:
\begin{align*}
n \sum_{\substack{j \neq k \\ j,k\in J}}
- \mathrm{cov}(\beta_j x_j, \beta_k x_k)
&=
n \Big(
\sum_{j \in J}\mathrm{var}(\beta_j x_j) -
 \mathrm{var}\Big(\sum_{j \in J}\beta_j x_j\Big)
 \Big)
\\
&=n \Big(
\sum_{j \in J}\mathrm{var}(\beta_j x_j)  - \mathrm{var}(\hat{y}) 
 \Big)
\\
&=
n \Big(\sum_{j \in J} \beta_j^2 \mathrm{var\ }x_j
- \frac{1}{n} |(\beta_0 + X\beta) - (\beta_0 +  \bar{X}\beta)|^2 \Big)
\\
&=
n \beta^T V \beta - |\underline{X}\beta|^2\\
&= \beta^T (n V - \underline{X}^T \underline{X}) \beta
\end{align*}
where $V$ is the $k\times k$ diagonal matrix of feature variances, $n$ is the number of data points in the original dataset. .

We may now use this expression to add a penalty to the least squares objective and minimize it for the Lagrangian $\lambda$:
\begin{align*}
J(\beta) &= |\underline{y} - \underline{X}\beta|^2 + 
\lambda \beta^T (n V - \underline{X}^T \underline{X}) \beta \\
\frac{\partial J}{\partial \beta}
&= -2\underline{X}^T\left(\underline{y} - \underline{X}\beta\right) + 2\lambda (n V - \underline{X}^T \underline{X})\beta \\
\frac{\partial^2 J}{\partial \beta^2} &= 
2\left((1-\lambda) \underline{X}^T \underline{X} + n\lambda V\right)
\end{align*}
The Hessian is positive definite Hessian for $\lambda<1$, hence a minimum exists.
\begin{align}
\label{eq:ablate-closed-form}
\hat{\beta}_{\mathrm{ccp}(\lambda)}(d) &= 
\left( (1-\lambda)\underline{X}^T \underline{X} + n \lambda V \right)^{-1} \underline{X}^T \underline{y}
\end{align}
This completes the proof.
\end{proof}

Next, we study inverted input dropout data augmentation.

\begin{theorem}[Inverted Input Dropout Data Augmentation as Regularization]
\label{thm:main2}
The Ordinary Least Squares (OLS) model with inverted input dropout data augmentation $\hat{\beta}(D^\lambda_N)$ converges to the least squares model with the Modified L2 penalty trained on the original data set $d$, i.e., $\hat{\beta}_{\mathrm{ml2p}(\lambda)}(d)$, as the augmented data set size $N \rightarrow \infty$.
\end{theorem}
\begin{proof}
Let $EX_j=\mu_j$. By conditioning on ablation, we observe that $E\tilde{X}_j=\mu_j$. Likewise, we find the following expectations for inverted input dropout (see section~\ref{sec:IID-expectations} of thee appendix) 
\begin{align*}
E\underline{\tilde{X}}^2
&= v_j + \frac{\lambda}{1-\lambda}(v_j + \mu_j^2)
\tag{where $v_j$ is variance of $j$-th feature}
\\
E\underline{\tilde{X}}_j \underline{\tilde{X}}_k
&=E\underline{X}_j \underline{X}_k
\\
E\underline{\tilde{X}}_j \underline{Y}
&=E\underline{X}_j \underline{Y}
\end{align*}
As in Theorem \ref{thm:main1}, averages over $D^\lambda_N$ converge almost surely to their averages over $d$ when $N \to \infty$ (Strong Law). Thus
\begin{align*}
\lim_{N\to \infty} \frac{1}{N} \underline{\tilde{X}}^T \underline{\tilde{y}}
    &\overset{a.s.}{\to} \frac{1}{n} \underline{X}^T \underline{y} \\
\lim_{N\to \infty} \frac{1}{N} \underline{\tilde{X}}^T \underline{\tilde{X}} 
    &\overset{a.s.}{\to} \frac{1}{n} \underline{X}^T \underline{X} + \frac{\lambda}{1-\lambda}D
\end{align*}
where $D$ is the diagonal matrix of $v_j + \mu_j^2$, in other words, the diagonal matrix of $EX_j^2$ for feature values.
Hence
\begin{align*}
\hat{\beta}(D^\lambda_N) &\overset{a.s.}{\to} \Big(X^TX + \frac{\lambda}{1-\lambda}D \Big)^{-1} X^T y
\end{align*}
\end{proof}

\begin{remark} [Computing and interpreting the Covariance Contribution Penalty]
\label{re:ccp}
The following remarks are meant to provide intuition to the reader in how the penalty behaves:
\begin{itemize}
\item As defined, the penalty involves $O(k^2)$ terms. This was done for clarity of interpretation. For computation, the following form is preferable:
\begin{align}
\label{eq:var-form}
n \Big(\sum_{j \in J} \mathrm{var}(\beta_j x_j)\Big)
- n \mathrm{var\ }(\hat{y})
\end{align}

\item If the features are uncorrelated, there is no penalty, and the solution is the OLS solution.

\item If the sum of contribution covariances is positive, the penalty itself becomes negative, and hence a reward. This happens when the contributions of individual features are mutually reinforcing rather than cancelling. The fact that contribution covariance can be negative makes it less intuitive as a penalty per se. Nonetheless there is limit to how much a reward this penalty will yield without increasing model error.

\item In a model with two almost perfectly-correlated features, the penalty is minimized, in this case most negative, when the model assigns equal contribution to the two features. This is desirable for model robustness.

\item The expression for $\hat{\beta}_{\mathrm{ccp}(\lambda)}(d)$
involves a convex combination of 
$\underline{X}^T \underline{X}$ and its diagonal  $n \lambda V$. Thus increasing $\lambda$ scales down the off-diagonal terms of
$\underline{X}^T \underline{X}$. When $\lambda = 1$, the off-diagonal terms go to zero and we end up with a solution corresponding to $k$ single-variable regressions
\begin{align*}
\lim_{\lambda \to 1} \hat{\beta}_{\mathrm{ccp}(\lambda)}(d)
&=V^{-1}\underline{X}^T\underline{y}
\end{align*}
\end{itemize}
\end{remark}

%\todo{I think we should compare with scaled L2 regularization}

\begin{remark}[Comparing Regularizations]
\label{re:compare}
The contribution covariance penalty has similarities with the $L_2$ penalty in that they both discourage large coefficients that tend to cancel each other. But they differ in specifics:
\begin{itemize}
\item 
While $L_2$ penalizes coefficients for being large in squared magnitude, contribution covariance only penalizes them if they tend to cancel each other out, i.e., their contribution is anticorrelated. So, for instance, there is no cancellation if the model has just one predictor, and hence the penalty is zero. Similarly, if the features are entirely uncorrelated with one another, i.e., $\underline{X}^T \underline{X} = n \lambda V$, the penalty is zero for all $\lambda$. In these cases, the solution penalized with contribution covariance is the same as the OLS solution, whereas the $L_2$ penalized solution is different.

\item In a model with two almost perfectly-correlated features, both contribution covariance and $L_2$ penalties are minimized when the model assigns equal attribution to the features. Nevertheless, $L2$ further shrinks the contributions while contribution covariance will not.

\item
Unlike L2, Modified L2 has the undesirable property that features whose means are farther from zero are penalized more. If feature means is zero, Modified L2 is equivalent to L2. With zero means, the only difference between Inverted Input Dropout and Ablated Augmentation is the rescaling of features in the former. This rescaling of features precisely cancels the decrease in off-diagonal terms of $\underline{X}^T\underline{X}$ that we get in Ablated Augmentation.
\end{itemize}
\end{remark}

\section{Extension to Neural Networks} 

For neural networks, unsurprisingly, it is hard to identify a closed-form for the model parameters, as we do for linear models in the proof of Theorem\ref{thm:main1}. Instead, we provide empirical results that mirror the result in the Theorem, i.e., we show that as we increase the strength of ablated data augmentation, the Contribution Covariance Penalty and the Modified L2 penalty falls across a variety of datasets and model depths (number of hidden layers). First, we identify the right notion of contribution for neural networks.

\subsection{From Contribution to Attribution}
\label{sec:attr}
Recall the definition of contribution from Section~\ref{sec:penalty}. The contribution of a feature $j$ to the prediction of an input $i$ is $\beta_j x_{ij}$; the sum of the contributions across the features, plus the bias term is the prediction.

To study deep networks, we need an analogous definition of contribution of a feature that takes the network's non-linearity into account. We can leverage the literature on feature-importance/attribution~\cite{STY17, Lundberg2017AUA, BSHKHM10,SVZ13,SGSK16,BMBMS16,SDBR14}. An attribution technique distributes the prediction score of a model for a specific input to its input features; the attribution to a input feature can be interpreted as its contribution to the prediction. We leverage a game theoretic approaches called Integrated Gradients~\cite{STY17}. The basic idea of Integrated Gradients is to accumulate the gradient along interpolation path between an information-less input called a baseline and the actual input. 

As defined, the attribution scores for the features sum to the prediction for a deep network just like contributions sum to predictions for a linear model. In fact, if you apply the two methods to a linear model, the attributions that are produced coincide with contributions. In what follows, we replace contributions with attributions in the Contribution Covariance Penalty.

\subsection{From Coefficients to Average Gradients}
Following the definition of contribution, the coefficient of a linear model for feature $j$ is the contribution divided by it's feature value $x_j$. We again replace contribution with attribution when looking for an analogous of coefficients. For Integrated Gradients, dividing attribution with the feature value (subtracted by baseline) is the average gradients along interpolated path. Again, if you apply IG with a zero baseline and divide the attribution by the feature value to a linear model, the resulted values are the coefficients. 
In what follows, we replace coefficients with attributions divided by feature value (subtracted by baseline) in the Modified L2 penalty.

\iffalse
\todo{show trends for IG and Shapley? Go back to the axioms? Show that the linear model case is identical to our math whether we do Shapley or IG. Assert that the eval is relative.}

\todo{Connection with dropout: Last layer dropout and look at weights}

\todo{naturally there are no L2 plots here}

\begin{figure}
    \centering
    \subfloat[Accuracy (L2)]{\includegraphics[width=0.4\linewidth]{./figures/linear_l2_acc}}
    \subfloat[Cancellation (l2)]{\includegraphics[width=0.4\linewidth]{./figures/linear_l2_cancel} } \\
    \subfloat[Accuracy (ablate)]{\includegraphics[width=0.4\linewidth]{./figures/linear_ablate_acc}}
    \subfloat[Cancellation (ablate)]{\includegraphics[width=0.4\linewidth]{./figures/linear_ablate_cancel} }
    \caption{Trend of accuracy and total cancellation as we increase regularization strength.}% 
    \label{fig:linear-diabetes}%
\end{figure}

\todo{put L2 and cancellation on the same plot?}
\todo{It appears that L2 is a better regularizer. Can we say why?}
\todo{make observations about increase in accuracy versus fall in cancellation}
\fi

\section{Experiments on Structured Data}

We study several classification and regression tasks and show that the Contribution Covariance Penalty and the Modified L2 Penalty drops with increased data augmentation. 

\subsection{Experimental Setup}
Our experiments are evaluated on a feed forward neural network with ReLU as activation function at each layer, with hidden layers of size 100. The number of hidden layers range from 0 (linear model) to 10. The models are trained with Adam and uses $1e^{-3}$ as initial learning rate for 200 epochs with a batch size of 256. Early stopping is applied if the loss on the validation set shows no improvement after 3 epochs and the best checkpoint is used for evaluation. For regression, the model is trained with mean squared error; for classification, the model is trained with cross entropy. For regression with no hidden layers, the closed form solution is used.

Each dataset is split with a 8:2 ratio for train and test. All evaluations such as MSE, accuracy, and penalty metric are evaluated on the test set. A quarter of training set is further used as validation except for linear regression when solved in closed form. Categorical features are first represented with one-hot encoding and all features are standardized to zero mean with unit variance. The response value is standardized to zero mean for regression. For classification, the value of the logit for the class is used as the response variable when calculating the Contribution Covariance Penalty. For Integrated Gradients, we select the all zero vector (i.e., each feature is set at its mean) as the baseline which represents a uninformative data point (refer to Section~\ref{sec:IG} in appendix for baselines). For IG, we pack 100 steps to approximate the Riemann sum. Notice that Modified L2 Penalty now coincides with L2 since we work with standardized features.

We run each experiment 10 times with different random seeds, which affects the parameter initialization for training, the ablation mask for ablated data augmentation, and the train-validation split.

\subsection{Datasets}

\textbf{FICO Score:}
The dataset is called Home Equity Line of Credit (HELOC) Dataset~\cite{FICO}. We predict \textit{ExternalRiskEstimate}, an estimate of loan default risk as judged by an analyst. There are 9861 data-points and 24 features.

\begin{figure*}[!t]
\centering
\includegraphics[width= 1.0\linewidth]{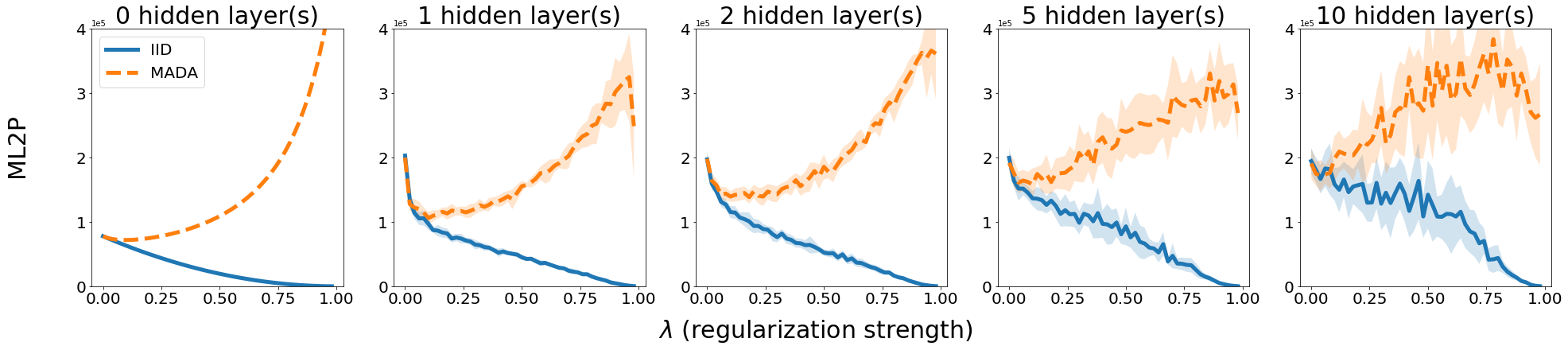}
\includegraphics[width= 1.0\linewidth]{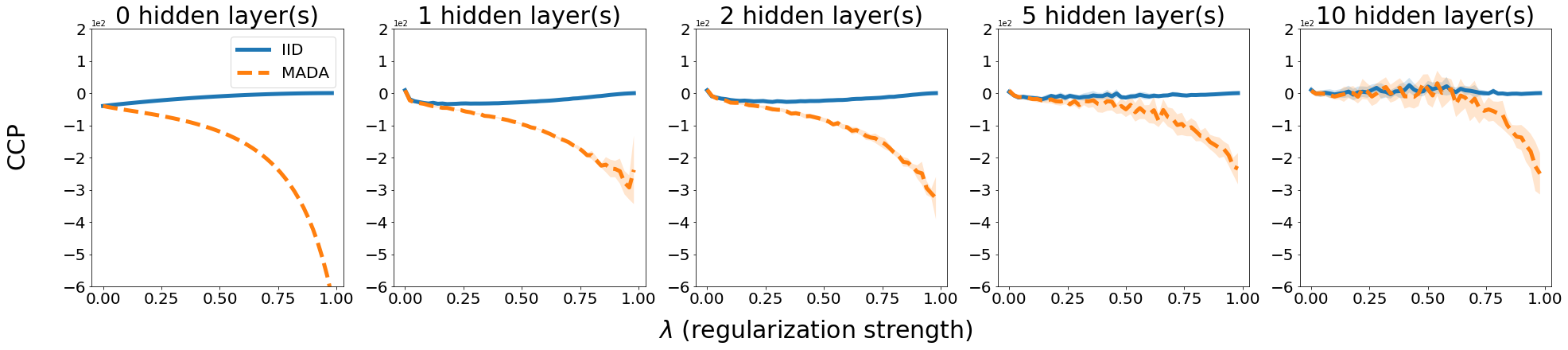}
\caption{\textbf{Fico Score: } Trend of ML2P calculated with AvgGradients and CCP calculated with IG when increasing $\lambda$ for MADA (Mean Ablated Data Augmentation) and IID (Inverted Input Dropout).}
\label{fig:fico}
\end{figure*}

\textbf{Census Income:}
The Adult Census Income dataset is collected from 1994 Census Database~\cite{Income}. It is a classification task; the response variable is the boolean condition 'does income exceeds \$50K'. The data set has 48842 data-points and 14 features.

\begin{figure*}[!t]
\centering
\includegraphics[width= 1.0\linewidth]{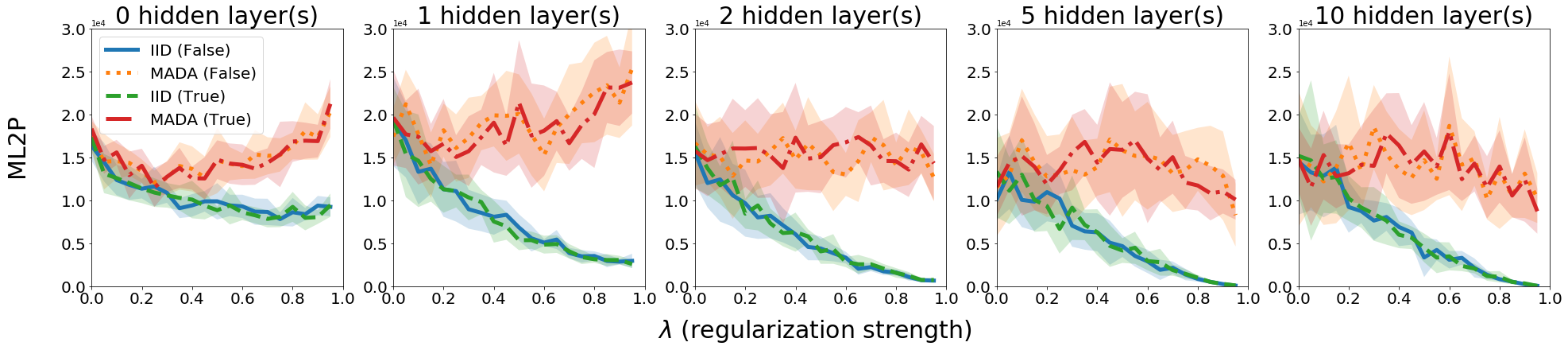}
\includegraphics[width= 1.0\linewidth]{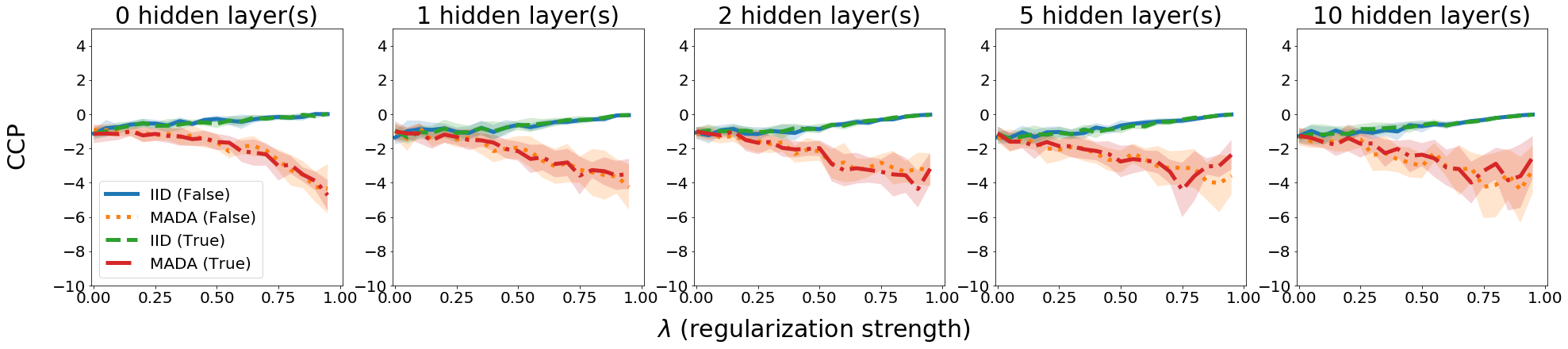}
\caption{\textbf{Census Income: }Trend of ML2P calculated with AvgGradients and CCP calculated with IG when increasing $\lambda$ for MADA (Mean Ablated Data Augmentation) and IID (Inverted Input Dropout) for both classes.}
\label{fig:census}
\end{figure*}

\textbf{Telco Customer Churn:}
This dataset contains customer data from a telecom company to predict if customer churns~\cite{Churn}. It has 7043 rows (customers) and 21 features; features include demographic, account, and subscription information.

\begin{figure*}[!t]
\centering
\includegraphics[width= 1.0\linewidth]{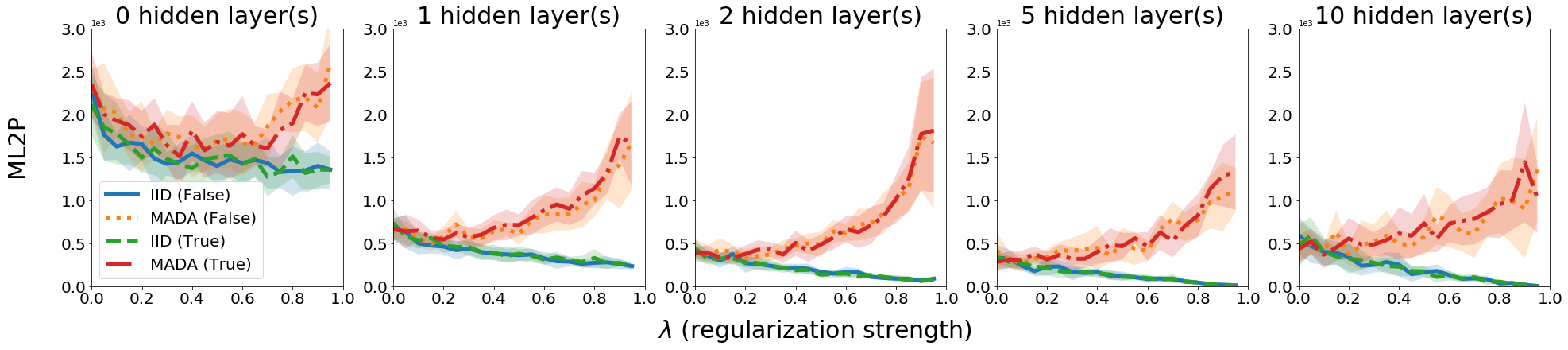}
\includegraphics[width= 1.0\linewidth]{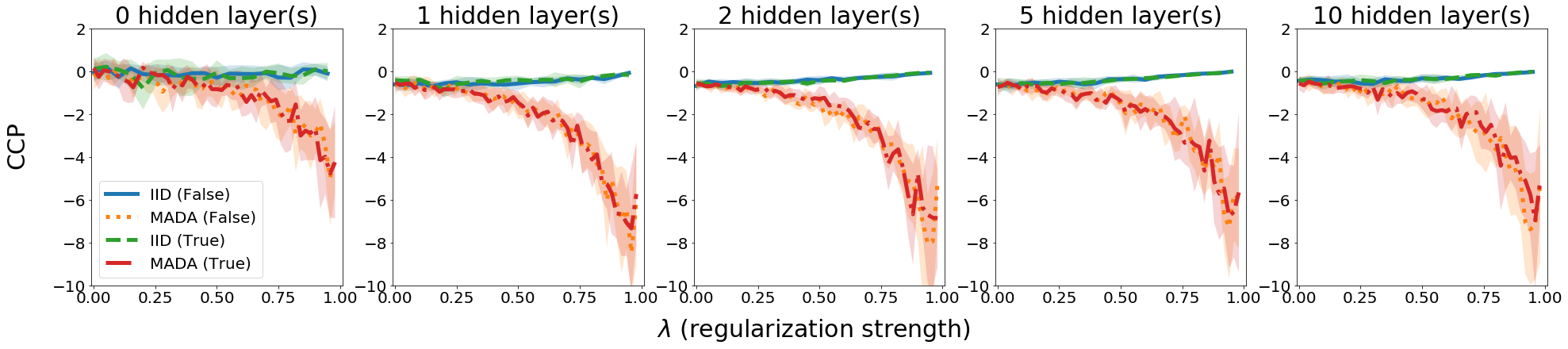}
\caption{\textbf{Telco Customer Churn:} Trend of ML2P calculated with AvgGradients and CCP calculated with IG when increasing $\lambda$ for MADA (Mean Ablated Data Augmentation) and IID (Inverted Input Dropout) for both classes.}
\label{fig:telco}
\end{figure*}

\subsection{Takeaways}
\label{sec:learnings}

We report how Modified L2 Penalty (ML2P) and Contribution Covariance Penalty (CCP) change as we increase regularization strength $\lambda$ for each dataset for different number of hidden layers in Figure \ref{fig:fico}, \ref{fig:census}, \ref{fig:telco} for both mean ablated data augmentation (MADA) and inverted input dropout (IID). For the three tasks, model performance has a downward trend as we increase $\lambda$ with some insignificant improvement in early $\lambda$ on the Census dataset. Like other regularization techniques, $\lambda$ is best searched with validation for model performance and we will mainly focus on the penalty trend. Figures of model performance trend can be found in appendix. Here are our takeaways:

ML2P decreases as we increases the regularization strength of IID; CCP decreases as we increases the regularization strength of MADA across all experiments regardless of being trained with closed form solution or applying augmentation during sgd training. Deep networks following the same trend shows that attribution methods, IG in particular, is a valid medium to extend findings from linear models as we would expect the models to behave somehow similar as they have similar performance and our theory applies.  

We also observed that ML2P increases as we increase the regularization strength for MADA. This aligns with Remark~\ref{re:ccp}. As $\lambda \rightarrow 1$, MADA becomes a single variable regression which results in larger ML2P. When $\lambda$ increases with IID, CCP is pushed towards 0. This is because when the coefficients are pushed towards 0, both variance of attribution and variance of prediction score goes to 0. As stated in Remark~\ref{re:compare}, the two augmentation techniques both discourage large coefficients but they differ in several aspects.  
    
\section{Conclusion}
Given the richness of theory on linear regression, and its relative paucity in ML, it would be valuable to find a way to apply the former to the latter. The typical bridge from linear theory to ML is the Taylor series expansion, in which a non-linear ML model is approximated as linear (or quadratic) in the neighborhood of a data point of interest. An alternate bridge is to use \emph{attributions} (see Section~\ref{sec:attr}); attributions compute an input specific linearly decomposition for deep networks. We use this bridge to show that properties of two data augmentation techniques that we can theoretically establish for linear models, persist empirically for neural networks, thereby improving our comprehension of neural networks.

\newpage
\section{Broader Impact}

This paper does not propose a new technique for data augmentation. It simply seeks to understand the effect of a specific form of data augmentation more deeply. In this sense, it may help ML developers reason about the effectiveness of data augmentation for various tasks. However, \emph{this work does not present any foreseeable societal consequence}, because any harm from bad data augmentation would have been discovered during the inspection of test accuracy.

\bibliographystyle{unsrtnat}
\bibliography{references}

\newpage
\appendix

\section{OLS regression with mean subtraction}
\label{sec:OLS-mean-subtraction}
Consider OLS regression on a data set $d = (X, y)$ where $y$ is the $n \times 1$ response and $X$ is an $n \times k$ matrix of numerical features. Categorical features are assumed to be coded as dummy variables, i.e., numerical variables that take on one of two values. Linear regression predicts $y$ as a linear combination of the columns of $X$, i.e. $\hat{y} =  \beta_0 + X\beta$. Prediction error is defined as $\epsilon = y - \hat{y}$, whose total square is minimized by least squares regression. Let bar denote the row vector of column-wise sample averages (or scalar for a column vector). Then
\begin{align*}
\label{eq:linear-regression}
\hat{\beta}(d) &= \arg \min_{\beta, \beta_0} |\epsilon|^2 \\
&= \arg \min_{\beta, \beta_0} |\epsilon - \bar{\epsilon}|^2 + n\bar{\epsilon}^2
\\
&= \arg \min_{\beta, \beta_0} |(y - \beta_0 - X\beta) - (\bar{y} - \beta_0 - \bar{X}\beta)|^2 + n(\bar{y} - \beta_0 - \bar{X}\beta)^2
\\
&= \arg \min_{\beta, \beta_0} |(y -\bar{y}) - (X-\bar{X})\beta)|^2 + n(\bar{y} - \beta_0 - \bar{X}\beta)^2
\\
&= \arg \min_{\beta}  |(y-\bar{y}) - (X - \bar{X})\beta|^2
   \tag{where $\beta_0 = \bar{y} - \bar{X}\beta$}\\
&= \arg \min_{\beta}  |\underline{y} - \underline{X}\beta|^2 \\
&= (\underline{X}^T \underline{X})^{-1}\underline{X}^T \underline{y}
\end{align*}
Here, underbar represent the subtraction of column-wise means, i.e., $\underline{z} = z - \bar{z}$.

\section{Ablation asymptotics}
\label{eq:ablation-asymptotics}
Now consider $\underline{\tilde{x}}_{\cdot j}$ and $\underline{\tilde{x}}_{\cdot k}$, the $j$-th and $k$-th column of the (mean-removed) feature matrix $\underline{\tilde{X}}$ and and $\underline{\tilde{y}}$ the response of $D_N^\lambda$.
\begin{align}
\frac{1}{N} \underline{\tilde{x}}_{\cdot j}^T \underline{\tilde{y}}
&= \frac{1}{N} \sum_{i=1}^N \underline{\tilde{x}}_{ij}\underline{\tilde{y_i}}
\nonumber
\\
\label{eq:pre-asymptotic}
&= \frac{1}{N} \sum_{i=1}^N \tilde{x}_{ij}\tilde{y_i}
- \left(\frac{1}{N} \sum_{i=1}^N \tilde{x}_{ij}\right)
\left(\frac{1}{N} \sum_{i=1}^N \tilde{y}_i\right)
\end{align}
Now consider the random variables $\tilde{X}_j$, $\tilde{X}_k$ and $j \neq k$ corresponding to the $j$-th and $k$-th features from a random data point of $D_N^\lambda$. These random variables are drawn from the original $d$ but ablated with independent probability $\lambda$. Let $Y$ denote the random variable of the response simultaneously drawn from $d$ but without ablation. Due to mean substitution, the expectation of each feature value in $D_N^\lambda$ is the same as in $d$. In other words, $E\tilde{X}_j=\bar{x}_j$. Likewise $E\tilde{X}_k=\bar{x}_k$ and $E Y=\bar{y}$. Applying the Strong Law of Large Numbers to the equation \ref{eq:pre-asymptotic}:

\begin{align*}
\lim_{N \to \infty}
\frac{1}{N} \underline{\tilde{x}}_{\cdot j}^T \underline{\tilde{y}}
&\overset{a.s.}{\to}
E\tilde{X}_j Y - E\tilde{X}_j EY
\tag{from \ref{eq:pre-asymptotic}}
\\
&= E(\tilde{X}_j - E\tilde{X}_j) (Y - EY)
\\
&= E(\tilde{X}_j - \bar{x}_j) (Y - \bar{y})
\\
&= \lambda E(\bar{x}_j - \bar{x}_j) (Y - \bar{y}) +
(1-\lambda) E(X_j - \bar{x}_j) (Y - \bar{y})
\tag{by conditioning the expectation on ablation}
\\
&= (1-\lambda) E(X_j - \bar{x}_j) (Y - \bar{y})\\
&= (1-\lambda) \frac{1}{n}\sum_{i=1}^n (x_{ij} - \bar{x}_j)(y_i - \bar{y})
\tag{direct evaluation on $d$}
\\
&=(1-\lambda)\frac{1}{n} \underline{x}_{\cdot j}^T \underline{y}
\end{align*}
Likewise
\begin{align*}
\lim_{N \to \infty}
\frac{1}{N} \underline{\tilde{x}}_{\cdot j}^T \underline{\tilde{x}}_{\cdot k}
&\overset{a.s.}{\to}
E\tilde{X}_j \tilde{X}_k - E\tilde{X}_j E\tilde{X}_k
\\
&= E(\tilde{X}_j - E\tilde{X}_j) E(\tilde{X}_k - E\tilde{X}_k)
\\
&= E(\tilde{X}_j - \bar{x}_j) (\tilde{X}_k - \bar{x}_k)
\\
&= (1-\lambda)^2 E(X_j - \bar{x}_j) (X_k - \bar{x}_k)
\\
&=(1-\lambda)^2\frac{1}{n} \underline{x}_{\cdot j}^T \underline{x}_{\cdot k}
\\
\lim_{N \to \infty}
\frac{1}{N} \underline{\tilde{x}}_{\cdot j}^T \underline{\tilde{x}}_{\cdot j}
&\overset{a.s.}{\to}
\mathrm{var\ }\tilde{X}_j
\\
&= E(\tilde{X}_j - E\tilde{X}_j)^2
\\
&= E(\tilde{X}_j - \bar{x}_j)^2
\\
&= (1-\lambda) E(X_j - \bar{x}_j)^2
\\
&=(1-\lambda)\frac{1}{n} v_j
\tag{population variance of $j$-th feature in $d$}
\end{align*}
Combining these results we have:
\begin{align*}
\lim_{N\to \infty} \frac{1}{N} \underline{\tilde{X}}^T \underline{\tilde{y}}
    &\overset{a.s.}{\to} \frac{1}{n} (1-\lambda) \underline{X}^T \underline{y} \\
\lim_{N\to \infty} \frac{1}{N} \underline{\tilde{X}}^T \underline{\tilde{X}} 
    &\overset{a.s.}{\to} \frac{1}{n} (1-\lambda)^2 \underline{X}^T \underline{X} + \lambda(1-\lambda) V
\end{align*}
where $V$ is the $k\times k$ diagonal matrix of feature variances. 

\section{Expectations for Inverted Input Dropout}
\label{sec:IID-expectations}
Let $EX_j=\mu_j$. By conditioning on ablation, we observe that $E\tilde{X}_j=\mu_j$. Now
\begin{align*}
E\underline{\tilde{X}}^2 &=
E\tilde{X}_j^2 - (E\tilde{X}_j)^2 \\
&= (1-\lambda)\frac{EX_j^2}{(1-\lambda)^2} - \mu_j^2\\
&= \frac{v_j + \mu_j^2}{1-\lambda} - \mu_j^2
\tag{where $v_j$ is variance of $j$-th feature}
\\
&= v_j + \frac{\lambda}{1-\lambda}(v_j + \mu_j^2)
\\
E\underline{\tilde{X}}_j \underline{\tilde{X}}_k
&= E\tilde{X}_j \tilde{X}_k - E\tilde{X}_j E\tilde{X}_k \\
&= (1-\lambda)^2 \frac{EX_jX_k}{(1-\lambda)^2} - EX_j EX_k\\
&= EX_j X_k - EX_j EX_k\\
&=E\underline{X}_j \underline{X}_k
\\
E\underline{\tilde{X}}_j \underline{Y}
&= E\tilde{X_j}Y - E\tilde{X_j}EY \\
&= (1-\lambda) \frac{EX_jY}{1-\lambda} - EX_jEY\\
&= EX_jY - EX_jEY\\
&=E\underline{X}_j \underline{Y}
\end{align*}

\section{Integrated Gradients}
\label{sec:IG}
Formally, suppose we have a
function $F: R^{n} \rightarrow [0,1]$ that represents a deep network.
Specifically, let $x \in R^{n}$ be the input at hand, and
$x' in R^{n}$ be a baseline input, meant to represent an informationless input. For vision networks, this could be the black image, or an image consisting of noise. For structured data, this is could be a vector of feature means. We consider the straight-line path (in $R^n$) from the baseline $x'$ to the input
$x$, and compute the gradients at all points along the path. Integrated gradients (IG) attributions are obtained by accumulating these gradients. The integrated gradient score for a feature $i$ with an input $x$ and baseline
$x'$ is:
\begin{equation}\label{eq:ig}
attrib_i = (x_i-x_i') \cdot \int_{\alpha=0}^{1} \tfrac{\partial F(x' + \alpha *( x- x'))}{\partial x_i}~d\alpha
\end{equation}

\iffalse
\section{Sampling-based Shapley Value}
\label{sec:SS}
Similar to the setup with Integrated Gradients where we have a function $F: R^{n} \rightarrow [0,1]$, input $x \in R^n$, baseline $x' \in R^n$. Additionally, we define $S$ as an ordered set of the features; $T \subseteq S!$; $P_i^{t}$ is the set of features which precede feature $i$ in $t \subseteq T$; a function $g: R^{n \times n \times m}  \rightarrow R^n$
\begin{align*}
g_j(x, x', P) =
\begin{cases}
    x_j,& \text{if } j \in P \\
    x'_j,              & \text{otherwise}
\end{cases}
\end{align*}
The Sampling-based Shapley value for a feature $i$ is:
\begin{equation}
    attrib_i = \frac{1}{|T|}\sum_{t \in T} \Big( F\big(g(x, x', P_i^{t} \cup {i}) \big) - F\big(g(x, x', P_i^{t})\big)\Big)
\end{equation}
\fi

\section{Model Performance Trend}
\begin{figure*}[h]
\centering
\includegraphics[width= 1.0\linewidth]{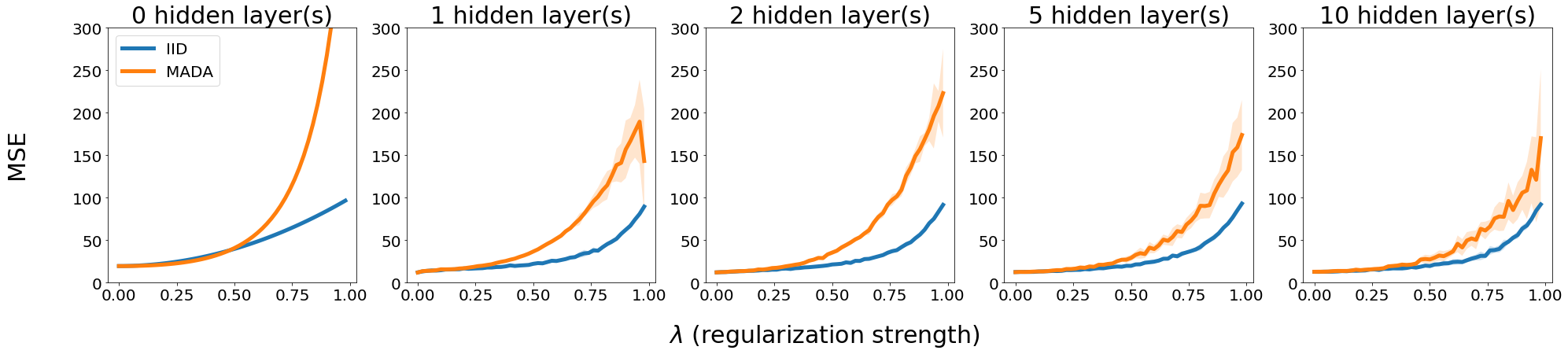}
\caption{MSE trend when increasing $\lambda$ on Fico.}
\label{fig:fico_mse}
\end{figure*}
\begin{figure*}[h]
\centering
\includegraphics[width= 1.0\linewidth]{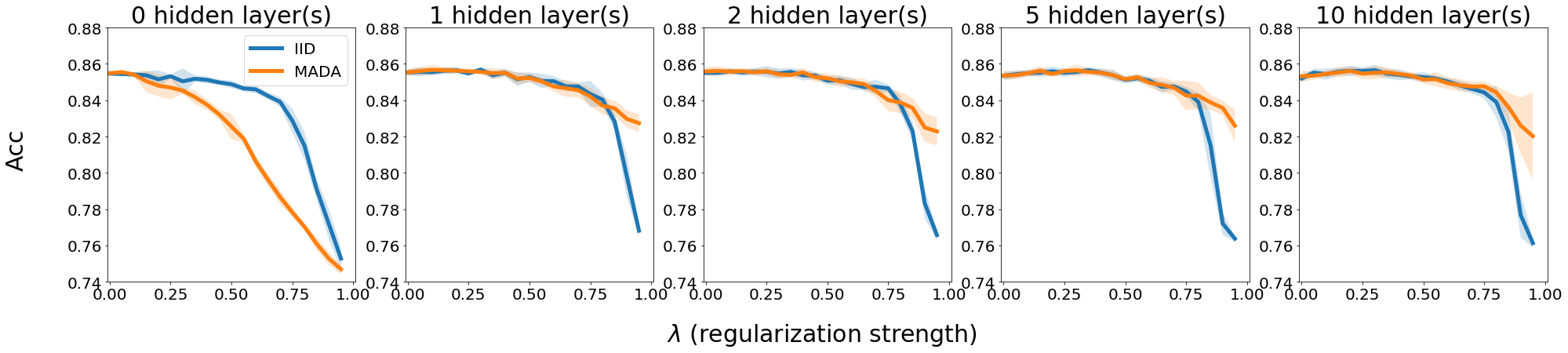}
\caption{Acc trend when increasing $\lambda$ on Census.}
\label{fig:adult_acc}
\end{figure*}
\begin{figure*}[h]
\centering
\includegraphics[width= 1.0\linewidth]{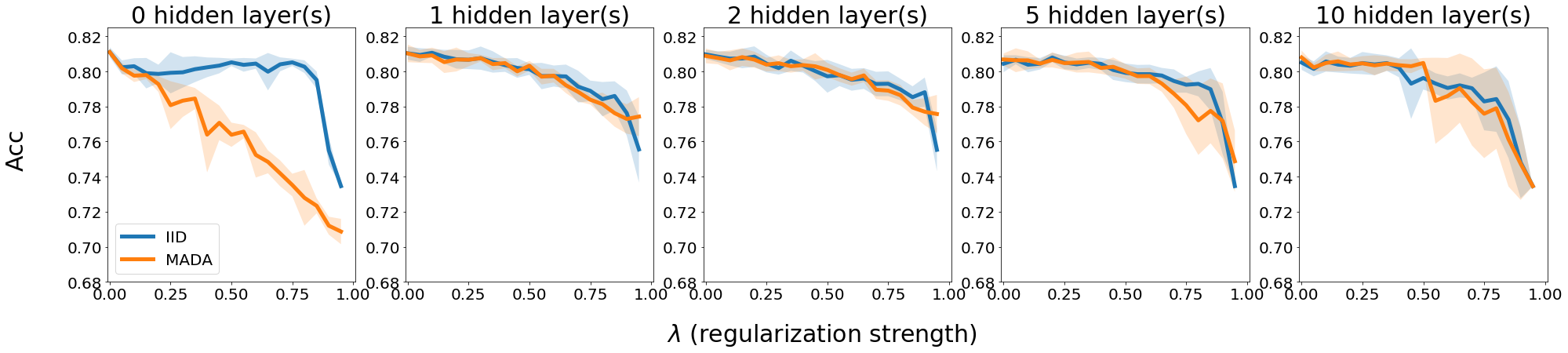}
\caption{Acc trend when increasing $\lambda$ on Telco.}
\label{fig:telco_acc}
\end{figure*}

\end{document}